\documentclass{article}

\usepackage{microtype}
\usepackage{graphicx}
\usepackage{subcaption}
\usepackage{multirow}
\usepackage{booktabs} % for professional tables

\usepackage{hyperref}

\usepackage[accepted]{arxiv}

\usepackage{amsmath}
\usepackage{amssymb}
\usepackage{mathtools}
\usepackage{amsthm}
\usepackage[most]{tcolorbox}

\usepackage[capitalize,noabbrev]{cleveref}

\theoremstyle{plain}
\newtheorem{theorem}{Theorem}[section]

\theoremstyle{definition}

\theoremstyle{remark}

\usepackage[textsize=tiny]{todonotes}

\usepackage{xcolor}

\newcommand{\methodname}{BALOR}
\newcommand{\causualmodel}{Llama3.2-1b}

\tcbset{
  colback=gray!5,
  colframe=black!40,
  boxrule=0.4pt,
  arc=2pt,
  left=4pt, right=4pt, top=4pt, bottom=4pt,
  fonttitle=\bfseries,
}

\icmltitlerunning{Faithful-Patchscopes}

\begin{document}

\twocolumn[
% \icmltitle{BALOR: Bias Alignment through Logit Recalibration for Explaining Hidden Representations of Large Language Models}
\icmltitle{Faithful-Patchscopes: Understanding and Mitigating Model Bias in Hidden Representations Explanation of Large Language Models }


\icmlsetsymbol{equal}{*}

\begin{icmlauthorlist}
\icmlauthor{Xilin Gong}{uga}
\icmlauthor{Shu Yang}{kaust}
\icmlauthor{Zehua Cao}{xxx}
\icmlauthor{Lynne Billard}{uga}
\icmlauthor{Di Wang}{kaust}

\end{icmlauthorlist}

\icmlaffiliation{uga}{University of Georgia}
\icmlaffiliation{kaust}{King Abdullah University of Science and Technology}
\icmlaffiliation{xxx}{Hong Kong Center for Construction Robotics}

% \icmlcorrespondingauthor{Lynne Billard}{first1.last1@xxx.edu}
% \icmlcorrespondingauthor{Di Wang}{first2.last2@www.uk}

\icmlkeywords{Machine Learning, ICML}

\vskip 0.3in
]



\printAffiliationsAndNotice{}  % leave blank if no need to mention equal contribution

\begin{abstract}
Large Language Models (LLMs) have demonstrated strong capabilities for hidden representation interpretation through Patchscopes~
\cite{ghandeharioun2024patchscopesunifyingframeworkinspecting}, a framework that uses LLMs themselves to generate human-readable explanations by decoding from internal hidden representations. However, our work shows that LLMs tend to rely on inherent linguistic patterns, which can override contextual information encoded in the hidden representations during decoding. For example, even when a hidden representation encodes the contextual attribute “purple” for “broccoli”, LLMs still generate “green” in their explanations, reflecting a strong prior association. This behavior reveals a systematic unfaithfulness in Patchscopes. To systematically study this issue, we first designed a dataset to evaluate the faithfulness of Patchscopes under biased cases, and our results show that there is an 18.84\% faithfulness decrease on average. We then propose \underline{B}ias \underline{A}lignment through \underline{Lo}git \underline{R}ecalibration (BALOR), which treats the output logits from an unpatched prompt as capturing model bias and contrasts them with logits obtained under patched contextual information. By recalibrating the logit distribution through this contrast, BALOR suppresses model bias and amplifies contextual information during generation. Experiments across multiple LLMs demonstrate that \methodname{} consistently outperforms existing baselines, achieving up to 33\% relative performance improvement.

\end{abstract}

\section{Introduction}
Large language models (LLMs) work by transforming input tokens into high-dimensional vectors and repeatedly feeding them through multiple layers of feed-forward transformations. The outputs of each feed-forward transformation are hidden representations that fundamentally govern how models encode attributes, reason over context, and ultimately generate predictions~\cite{orgad2025llms,DBLP:journals/corr/abs-2507-10155}. Thus, understanding what information is stored and how it is used inside these hidden representations is essential for building trustworthy LLMs. Hidden representation explanation methods, such as linear probing~\cite{alain2018understandingintermediatelayersusing, belinkov-2022-probing,DBLP:journals/tacl/ChenWYZWLWYCW25,DBLP:journals/corr/abs-2511-06419,DBLP:journals/corr/abs-2510-10205}, train a linear classifier to analyze whether content can be extracted from hidden representations, and causal interventions~\cite{vig2020causalmediationanalysisinterpreting, meng2023locatingeditingfactualassociations,DBLP:journals/corr/abs-2508-02087,DBLP:conf/acl/YaoYXHLW25,DBLP:conf/emnlp/HongZHZ0Y24,DBLP:conf/icml/0003LKCH025} change keywords in the input prompts to analyze the changes in hidden representations. 
More recent work, Patchscopes~\cite{ghandeharioun2024patchscopesunifyingframeworkinspecting} explains hidden representations by using the LLMs themselves as interpreters. As illustrated in Figure~\ref{fig:introexample}, a source prompt provides contextual information describing broccoli as purple. Then, the hidden representation of broccoli, which corresponds to the purple attribute value, is extracted. This hidden representation is then patched into a target prompt by replacing the hidden representation at a placeholder token $x$ during the target prompt’s inference process. Conditioned on the patched broccoli's hidden representation, the model generates a natural-language response that serves as an explanation of the information encoded in the patched hidden representation.

\begin{figure}[t]
    \centering
    \includegraphics[
        width=\linewidth,
        trim=70 62 80 70,
        clip
    ]{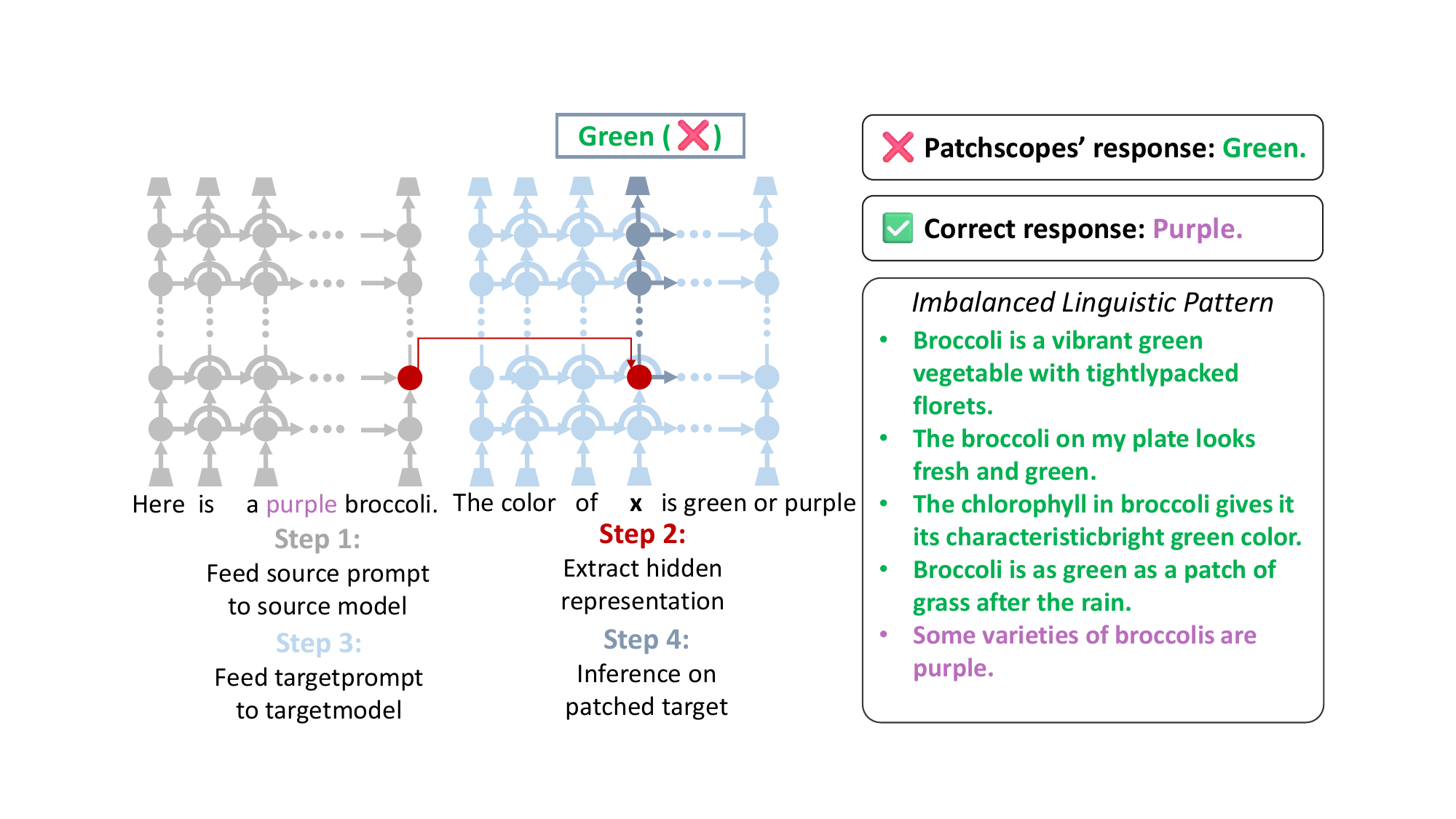}
    %\vspace{-6pt}
    \caption{Overview of the Patchscopes framework and an example illustrating its unfaithfulness under an imbalanced linguistic pattern. The patched hidden representation encodes the contextual attribute ``purple'' for ``broccoli'', but under Patchscopes' framework, LLMs still generate ``green'' in their explanations.}
    \label{fig:introexample}
    %\vspace{-10pt}
\end{figure}

Despite the promising nature of the Patchscopes framework, it cannot guarantee that the content generated for explanation faithfully reflects the information contained in the hidden representations~\cite{lyu-etal-2024-towards, bhan2025didifaithfullysay, hu2025monicarealtimemonitoringcalibration, yang2026investigatingcotmonitorabilitylarge}. Our work observed that the generated explanation content is heavily affected by the model's inherent bias, which is induced by imbalanced linguistic patterns and prevents models from faithfully reflecting the contextual information in their hidden representations. For example, as Figure~\ref{fig:introexample} shows, in most linguistic patterns, broccoli is associated with the green color attribute values, while only a few are associated with purple. When we want to explain the color-related hidden representations of broccoli encoded with this less-frequent purple attribute, Patchscopes ignores the purple information in this hidden representation and instead generates the ``green'' attribute.
While Figure~\ref{fig:introexample} provides an illustrative example of this phenomenon, we further validate this effect through a systematic controlled analysis in Section~\ref{sec:casual_analysis}. We demonstrate that the faithfulness of the Patchscopes framework is distorted by this model bias, which leads to an 11.84\% to 28.47\% decrease in explanation faithfulness. This observation highlights a fundamental challenge in Patchscope: during generation, LLMs often rely on high-frequency associations learned from imbalanced linguistic patterns instead of the contextual information encoded in their hidden representations. %This unfaithfulness is a significant limitation of the current Patchscopes framework because we cannot guarantee that the explanation content genuinely reveals what the model encodes, rather than merely reproducing model bias.}

It is notable that current debiasing methods are difficult to integrate directly into the Patchscopes framework, making a simple ``debiasing-then-patching'' pipeline impossible. Supervised finetuning (SFT) requires large annotated datasets to capture both biased and unbiased instances~\cite{lin2022truthfulqa}, which can be very expensive. Also, SFT alters model parameters globally, thereby changing the original hidden representations and undermining Patchscopes’ goal of faithfully explaining what the original model encodes. Prompt-based debiasing methods~\cite{amplayo-etal-2023-query} are training-free, but they depend heavily on prompt designations. The results are unstable and vary with minor lexical or syntactic changes. Steering vector intervention-based methods~\cite{li2025fairsteerinferencetimedebiasing} attempt to remove or steer bias directions within latent space by identifying and manipulating subspaces of internal representations. However, these approaches assume that model bias can be isolated into well-defined and linearly separable directions, and the exact directions corresponding to model bias are unknown and hard to estimate reliably. 

% These conditions are rarely satisfied since we cannot acquire the model bias direction exactly. \di{rewrite, what does this mean?} 

In order to overcome these limitations, we propose \underline{B}ias \underline{A}lignment through \underline{Lo}git \underline{R}ecalibration (\methodname{}), a logit-recalibration decoding method that directly targets the source of distortion, which is motivated by the knowledge that logit-level probability distributions directly determine text generation and reflect the influence of model bias. By subtracting logits obtained from an unpatched contrastive prompt, internal evidence is contrastively separated from empirical priors in \methodname{}. This approach amplifies contextual information encoded in the patched representation while suppressing undesired model bias. In summary, this paper makes three key contributions:
\begin{itemize}
    \item We designed a Q\&A dataset according to linguistic pattern-induced model bias. Our analysis demonstrates that model bias fundamentally limits the faithfulness of hidden-representation explanations in LLMs.
    \item We introduce a logit-recalibration decoding method \methodname{} to mitigate model bias at the inference stage.
    \item We provide extensive experiments across four modern LLMs, demonstrating that \methodname{} substantially improves explanation faithfulness by up to 33\% while remaining robust across hyperparameters and decoding temperatures.
\end{itemize}

\section{Preliminaries}
We follow the framework and notation of Patchhscopes~\cite{ghandeharioun2024patchscopesunifyingframeworkinspecting}.
Let \( S = \langle s_1, \ldots, s_n \rangle \) be a source prompt of length \(n\), and \(M\) be a source LLM with \(L\) layers.
Feeding \(S\) to \(M\) yields hidden representations \( \mathbf{h}^{(\ell)}_i \in \mathbb{R}^d \), where \(i \in \{1,\ldots,n\}\) indexes token positions and \( \ell \in \{1,\ldots,L\} \) indexes layers.
A source representation to be inspected is specified by the tuple \( (S, i, M, \ell) \), corresponding to \( \mathbf{h}^{(\ell)}_i(S) \).

To decode information from this representation, Patchscopes defines a separate target execution.
Let \( T = \langle t_1, \ldots, t_m \rangle \) be a target prompt and \(M^*\) be a target LLM with \(L^*\) layers.
We denote by \( \mathbf{h}^{(\ell^*)}_{i^*}(T) \) the hidden representation at target position \( i^* \in \{1,\ldots,m\} \) and layer \( \ell^* \in \{1,\ldots,L^*\} \) during the forward pass of \(M^*\) on \(T\).

Given an optional mapping function \( f : \mathbb{R}^d \rightarrow \mathbb{R}^{d^*} \), the patching operation intervenes on the target computation by replacing
\[
\mathbf{h}^{(\ell^*)}_{i^*}(T) \;\leftarrow\; f\!\left(\mathbf{h}^{(\ell)}_i(S)\right),
\]
while keeping all other hidden representations unchanged. In this work, we restrict our attention to the case where the source and target models are the same. Consequently, the mapping function $f(\cdot) = \mathrm{id}(\cdot)$.
The modified hidden representations are then propagated through the remaining layers of \(M^*\), producing a patched output distribution over the next token:
\[
p_{\text{patched}}(w_{t+1} \mid T)
= \sigma\!\left(W_o \mathbf{h}^{(L^*)}_t(T) + \mathbf{b}_o\right),
\]
where $\sigma$ denotes the softmax function, \(W_o\) and \(\mathbf{b}_o\) denote the output projection parameters of \(M^*\).

By construction, this intervention halts further contextualization from the source sequence \(S\); thus, any information reflected in the patched generation must be recoverable from the source representation \( \mathbf{h}^{(\ell)}_i(S) \) via the post-patching computation.
%Patchscopes therefore reveal the information encoded in hidden representations by analyzing their causal effect on generation under a controlled target prompt and model configuration.

\section{Inherit Model Bias will Affect Explanation Faithfulness}\label{sec:casual_analysis}
Although the Patchscopes framework is widely used, whether its use is a faithful method remains a problem. Building on the insight that LLMs exhibit model bias, we need to determine how much faithfulness is affected by this. In this section, we design a dataset to analyze the causal relationship between linguistic patterns, model bias, and the faithfulness of hidden representation explanations.  
\subsection{Dataset}\label{sec:dataset}
There are four tasks in this dataset: color, gender, culture, and age. For each datapoint, there will be one $noun$ to be explained and two attribute values $a_{pri}$ and $a_{sec}$ for the model to select. Table~\ref{tab:overview_4tasks} provides an overview of these four tasks. The source and target prompt examples of each task are provided in Appendix~\ref{app:example_task}. The dataset consists of two parts: a biased part and a non-biased one. For datapoints in the biased part, the model exhibits a clear preference toward $a_{pri}$ in the absence of additional contextual information, whereas no such preference is observed for datapoints in the non-biased subset. We use $\mathcal{D}$ for each task in each part.

\textbf{Biased Subset}
The biased subset contains datapoints from all four tasks.
For the color task, we use the category names of the Large Vocabulary Instance Segmentation dataset~\cite{gupta2019lvis} for nouns as item names. Then, the color attribute values related to these nouns are extracted from the RedPajiama subset~\cite{weber2024redpajama}, which is described as the best-effort reproduction of the Llama training dataset. Thus, it can reflect the linguistic patterns in natural language. The color with the highest co-occurrence frequency is set as $a_{pri}$, and the color with the second highest co-occurrence frequency is set as $a_{sec}$.
The nouns are then selected such that the model exhibits a clear preference for $a_{pri}$ in the biased dataset. The details of the selection strategy are provided in Appendix~\ref{app:color_dataset}.
For gender, culture, and age tasks, we extract the nouns and two attribute values from the dataset provided by~\cite{hernandez2024linearity}. We define the model-preferred attribute identified by~\citet{hernandez2024linearity} as $a_{pri}$, and the other one as $a_{sec}$. More details are provided in Appendix~\ref{app:biased_dataset}. 

\textbf{Non-biased Dataset} The non-biased dataset only includes the color task. The datapoints of the color task that are not selected for the biased dataset construct non-biased dataset.

\begin{table}[]
    \centering
    \resizebox{0.8\linewidth}{!}{
    \begin{tabular}{l|ccc}
        \toprule
         Task&Noun& Attribute value  \\
         \midrule
         Color & item name & Color1, Color2 \\
         \midrule
         \multirow{4}{*}{Gender}
         & degree &  He, She\\
         & characteristic & He, She \\
         & person name  &He, She \\
         & occupation & He, She\\
         \midrule
         \multirow{2}{*}{Culture}
         & person name & Country1, Country2\\
         & person name & Religion1, Religion2 \\
         \midrule
         Age & occupation &  Young, Old \\
         
         \bottomrule
    \end{tabular}}
    \caption{Details for each task in the dataset.}
    \label{tab:overview_4tasks}
\end{table}

\subsection{Experiment Settings}\label{sec:casual_expsetting}
\textbf{Dataset and Models} We use the biased and non-biased subsets of the color task to analyze the causal relationship between linguistic patterns, model bias, and Patchscopes faithfulness. Due to the larger size of the non-biased dataset, we perform five repeated experiments with undersampling. Four models are used for experiments: Llama2-7b-chat~\cite{touvron2023llama2openfoundation}, Llama3.2-1b-Instruct, Llama3-8b-Instruct~\cite{grattafiori2024llama3herdmodels}, and Qwen3-4b-Instruct-2507~\cite{yang2025qwen3technicalreport}.

\textbf{Prompts}
The source prompt ($S_i$) is formatted as: \texttt{Here is an \{a$_{i,sec}$\} \{noun$_i$\}.} The target prompt ($T_i$) is formatted as: \texttt{The color of x is \{a$_{i_pri}$\} or \{a$_{i_sec}$\}?} During patching, the position of \texttt{x} in the target prompt defines the target index, while the hidden representations at the \texttt{noun}’ token indices in the source prompt are extracted for interpretation. For simple implementation, we repeat the placeholder \texttt{x} to match the token length of the \texttt{noun}. To mitigate option position bias, we additionally use a counterpart target prompt with the attribute options swapped and report performance averaged over both prompts.

\textbf{Evaluation Metric}
We use Show Rate (SR $\Uparrow$) to evaluate the faithfulness of Patchscopes; SR measures the ratio of responses that LLM generates $a_{sec}$, the attribute value aligns with contextual information, as an explanation. The SR is 
\[
SR = 
\frac{1}{|\mathcal{D}|}
\sum_{\mathcal{D}_i \in \mathcal{D}}
\mathbb{I}\!\left[
M(T_i) = a_{i, sec}
\right]. 
\] 
\subsection{Analysis}\label{sec:casual_analysis}
% We notice that LLM will show strong model bias on some datapoints, for example, it will constantly think that broccoli is green instead of purple when no additional information beyond the noun is provided, although purple broccoli does exist.
In this section, we aim to determine the relationship between linguistic patterns, model bias, and Patchscopes faithfulness.

First, we examine whether model bias arises from the imbalanced linguistic patterns, which are the co-occurrence frequencies between color names and item names here~\cite{brill2024neural, kiyomaru-etal-2024-comprehensive-analysis}. For each noun-color pair, we compute the frequency difference $\Delta f$, while $\Delta f=f_{pri}-f_{sec}$ where $f_{x}$ denotes the co-occurance frequency between $a_{x}$ and $noun$, $x\in \{pri, sec\}$.
Take \causualmodel{} as an example, across five repeated runs, we fit a logistic regression with $\Delta f$ standardized to one standard deviation. The logit coefficient, corresponding odds ratio (OR), and ROC-AUC are used to measure predictive performance. The mean value and $95 \%$ confidence interval of each metric are shown in Table \ref{table:casual_analysis}. It shows that every $+1$ standard deviation increase in $\Delta f$ increases the odds of exhibiting bias by a factor of 9.338 on average, and the ROC-AUC value of 0.619 also shows that the discriminative power suggests that an imbalanced linguistic pattern is a meaningful determinant of model bias. The causal analysis results between linguistic patterns and model bias on additional models are provided in the Appendix~\ref{app:casual}.
\begin{table}[t]
\centering
\resizebox{\linewidth}{!}{
\begin{tabular}{lcc}
\toprule
\textbf{Metric} & \textbf{Mean Value} & \textbf{95\% CI} \\
\midrule
Logit coeff (+1 SD) & 2.234 & (1.487, 2.981) \\
OR (+1 SD)        & 9.338 & (4.425, 19.706) \\
ROC-AUC                       & 0.619 & (0.580, 0.655) \\
\bottomrule
\end{tabular}}
\caption{Repeated undersampling results.}\label{table:casual_analysis}
\end{table}

\begin{figure}[t]
    \centering
    \includegraphics[width=\linewidth]{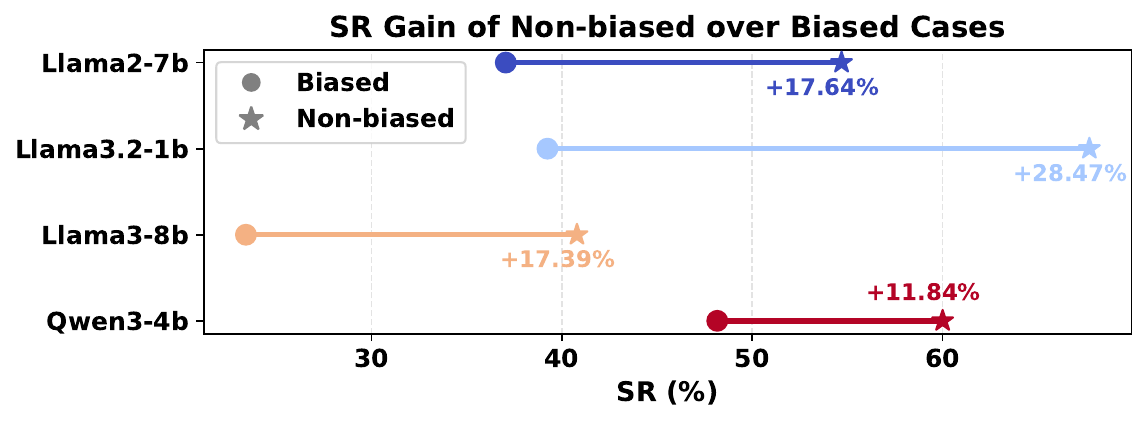}
    \caption{The result of attribute SR for vanilla Patchscopes on biased and non-biased datasets.}
    \label{fig:bias_vs_nobias}
\end{figure}

Then we test the relationship between model bias and Patchscopes'~
\cite{ghandeharioun2024patchscopesunifyingframeworkinspecting} faithfulness on biased and non-biased data. The results are shown in Figure~\ref{fig:bias_vs_nobias}. The performance gap between biased and non-biased datasets shows that model bias indeed undermines the faithfulness of Patchscopes.
%the hidden representation explanation's abilities, which makes this LLM explainability tool not faithful enough under biased conditions. 

%.For \causualmodel{}, biased data which contains 344 cases yield only 135 correct explanations, with $Acc=39.2\%$. In contrast, non-biased data, which contains 802 cases achieve 543 correct explanations, with $Acc=67.7\%$. 
%Thus, the 28.5\% performance gap shows that bias indeed undermines the hidden representation explanation's abilities, which makes this LLM explainability tool not faithful enough under biased conditions. 

\section{Methodology}
\subsection{Locate the Bias-sensitive Layer}

% Identifying the appropriate source and target layers is crucial in \methodname{}, as the degree of bias encoded or manifested in hidden representations varies across the model’s depth. Bias often emerges in specific middle layers where the model begins committing to high-frequency associations learned from training~\cite{lepori-etal-2025-racing}. Therefore, locating these bias-sensitive layers allows patching to intervene where empirical priors dominate reasoning. Consider the source and target model with $n$ layers; there can be $n\times n$ possibilities of a source-target-layer pair. If we exhaustively enumerate to find the optimal layer pair, the time complexity is $O(n^2)$, which would be expensive for LLM inference. Although~\cite{ghandeharioun2024patchscopesunifyingframeworkinspecting} suggested selecting the same layer when source and target models are the same, the time complexity to find the optimal layer pair is still $O(n)$. To efficiently locate the bias-related layer, we design quantitative metrics that indicate how strongly a layer encodes and utilizes the biased attribute.

\citet{lepori-etal-2025-racing} finds that model bias is not uniformly distributed across layers but gradually emerges in intermediate layers. We refer to these intermediate layers, where biased priors are actively encoded and start to influence downstream computation, as bias-sensitive layers. After these layers, model predictions tend to collapse onto empirical priors, making later interventions increasingly ineffective. Thus, locating the bias-sensitive layer is essential for effective bias mitigation and for enabling faithful Patchscopes explanations.

% , as bias is unevenly distributed across a model’s depth and often consolidates in intermediate layers where high-frequency associations begin to dominate reasoning~\cite{lepori-etal-2025-racing}. Once such empirical priors are fully established, intervening at later layers becomes increasingly ineffective.

For a model with $n$ layers, exhaustively searching all source–target layer pairs incurs $O(n^2)$ complexity, which is expensive for LLM inference. Even restricting source and target to the same layer as~\citet{ghandeharioun2024patchscopesunifyingframeworkinspecting} suggests, it still requires an $O(n)$ scan. We therefore design efficient quantitative metrics to identify bias-sensitive layers.

\textbf{Logit Difference.} To locate bias-sensitive layers, patching should be applied only after the attribute has been encoded in the model’s internal representations. We therefore assess whether the model’s representations already differentiate the context-aligned attribute from the biased one at each layer. Thus, we use logit-lens~\cite{elhage2021mathematical} to measure the logit difference (LD) between the secondary (context-aligned) and primary (biased) attributes at layer $l$. A larger positive difference indicates that the layer encodes evidence favoring the contextual attribute. Implementation details of this procedure are provided in Appendix~\ref{app:logit_diff}. We quantify this effect by averaging the LD over dataset $\mathcal{D}$.
\begin{equation*}
LD_l = \frac{1}{|\mathcal{D}|}\sum_{\mathcal{D}_i \in \mathcal{D}}
\left[
  (h^{(l)} W_o^{\top})_{a_{i,\mathrm{sec}}}
  -
  (h^{(l)} W_o^{\top})_{a_{i,\mathrm{pri}}}
\right]. 
\end{equation*}
\textbf{Gradient Similarity Alignment.} If patching is applied at layers close to the output, the model may have already committed to empirical priors, making bias difficult to override~\cite{lepori-etal-2025-racing}. To locate bias-sensitive layers where patched information still affects model decisions, we measure the alignment between the output gradient and the attribute direction, which indicates whether the model relies on biased or contextual information.

% However, if the layer on which we performed patching is too close to the last layer, the model may have already collapsed onto its empirical priors, making the patched signal difficult to detect~\cite{lepori-etal-2025-racing}. The information contained in the patched token $h^{(l)}(n)$ should be propagated and utilized by the deeper layers of the target model, rather than the model continuing to rely on its prior internal representations. If the model's output remains insensitive to $h^{(l)}(n)$, the gradient along this direction will be nearly zero, otherwise, the gradient will shift along the direction encoding this attribute. Thus, we use the alignment between gradient in the target model and the attribute direction in the patched token to measure how much the model is using the patched information. 

A linear probe is trained to find the attribute direction $w$. More details about this are shown in Appendix~\ref{app:linear_probing}. Then we patch hidden representation $h^{(l)}(n)$ into target prompt $T$, the probability to generate correct color $c$ is $p(c\mid T^l)$. Let the gradient of the patched hidden representation be $g_i$, then the gradient similarity alignment (GSA) score for layer $l$ is defined as the dataset-averaged absolute cosine similarity between the gradient direction and the attribute direction:
\[
GSA_l
=
\frac{1}{|\mathcal{D}|}
\sum_{\mathcal{D}_i \in \mathcal{D}}
\left|
\frac{g_i^{\top} w}{\|g_i\|_2 \, \|w\|_2}
\right|,
\quad
g_i = \frac{\partial \log p(c\mid T^l)}{\partial h^{(l)}(n)}.
\]

\textbf{Layer Selection.}
After analyzing the individual effectiveness of LD and GSA, we integrate the two metrics to compute a final score for each layer. To justify their combination, we analyze the statistical relationship between the two metrics. Using \causualmodel{} as an example, we compute the Spearman correlation coefficient~\cite{spearman1961proof} between LD and GSA, obtaining $\rho=-0.147$ with $p<$ 0.001, which indicates a weak correlation. To further characterize their relationship, we fit an isotonic regression~\cite{miles1959complete}, shown as the solid blue curve in Figure~\ref{fig:correlation_Llama3.2_1b}. The deviation from this idealized curve further suggests that LD and GSA are complementary, supporting their joint use in layer selection. The definitions of Spearman correlation and isotonic regression are provided in Appendix~\ref{app:spearman} and~\ref{app:Isotonic}. The correlation analysis for additional models is reported in Appendix~\ref{app:correlation}.
After normalizing $LD_l$ and $GSA_l$ to the $[ 0, 1]$ range, we combine them into a single layer-selection objective score and choose the optimal patching layer as
\[
l^{\star}
=
\arg\max_{\,l \in \mathcal{L}}
\Bigl(
w\,LD_l + (1-w)\,GSA_l
\Bigr),
\]
where $w$ controls the relative importance of representation-level encoding and gradient-level utilization.
We set $w=0.8$ for all models, which yields the best empirical performance; the selection of $w$ is determined via grid search, with results reported in Appendix~\ref{app:parameter}, and the layer scores and best layers for four models are shown in Figure~\ref{fig:best_layer}.

\begin{figure}[t]
    \centering
    \includegraphics[width=0.85\linewidth]{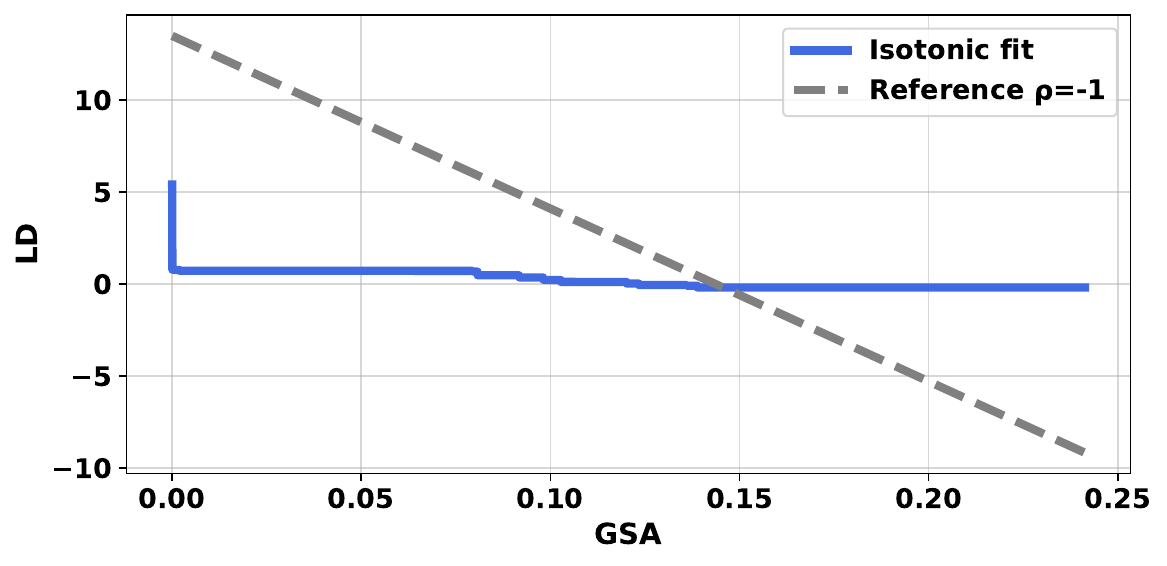}
    %\vspace{-6pt}
    \caption{Isotonic Regression of GSA and LD on \causualmodel{}
    }
    %\vspace{-8pt}
    \label{fig:correlation_Llama3.2_1b}
\end{figure}

\begin{figure}[t]
    \centering
    \includegraphics[width=0.85\linewidth]{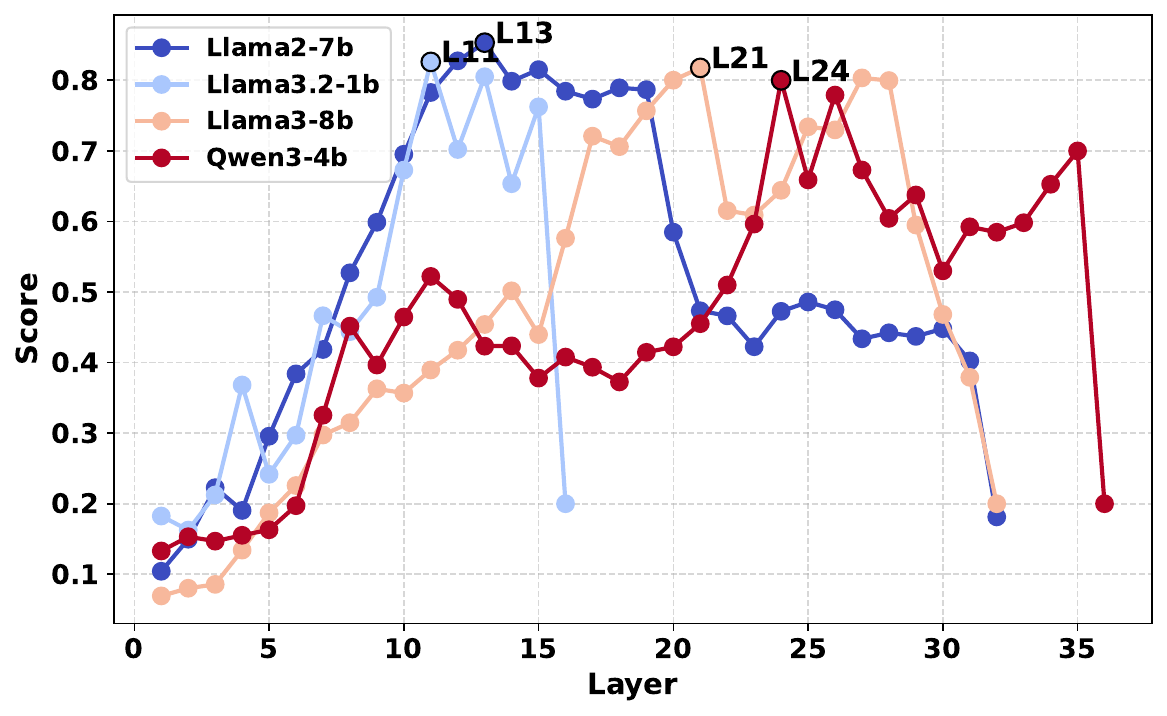}
    %\vspace{-6pt}
    \caption{The selected best layer of each model for patching. Layer index starts from 1 with $w$=0.8.}
    \label{fig:best_layer}
    %\vspace{-8pt}
\end{figure}

% \subsection{Can model bias Be Mitigated?}
% We first employ a prompt-guidance intervention method to explore whether the model bias can be controlled. We designed the following 3 prompts and added them as a prefix of target prompt:

% \begin{tcolorbox}
% \textbf{Prompt1}: Do not use outside or empirical knowledge, directly answer. \textcolor{blue}{\textless target prompt \textgreater}.

% \textbf{Prompt2}: Use only internal evidence from the input, exclude background knowledge. \textcolor{blue}{\textless target prompt \textgreater}.

% \textbf{Prompt3}: Answer strictly from the given data, avoid any external reasoning. \textcolor{blue}{\textless target prompt \textgreater}.
% \end{tcolorbox}

% As shown in Figure~\ref{fig:prompt_result}, prompts-guidance cases generally outperform the baseline, although in Qwen3-4b, the guidance of \textit{prompt2} and \textit{prompt3} has a slightly lower accuracy. These result indicates that prompts can steer the model to inference based on the information encoded in patched hidden representations, instead of relying on empirical bias and knowledge. This tendency means that model bias can be mitigated by inference-time-steering.

% \begin{figure}[t]
%     \centering
%     \includegraphics[width=\linewidth]{figure/prompt.png}
%     \caption{The result of prompt-guidance of each model.
%     }
%     \label{fig:prompt_result}
% \end{figure}

\subsection{\methodname{}}

\begin{figure}[t]
    \centering
    \includegraphics[width=\linewidth,trim=40 10 120 40,
        clip]{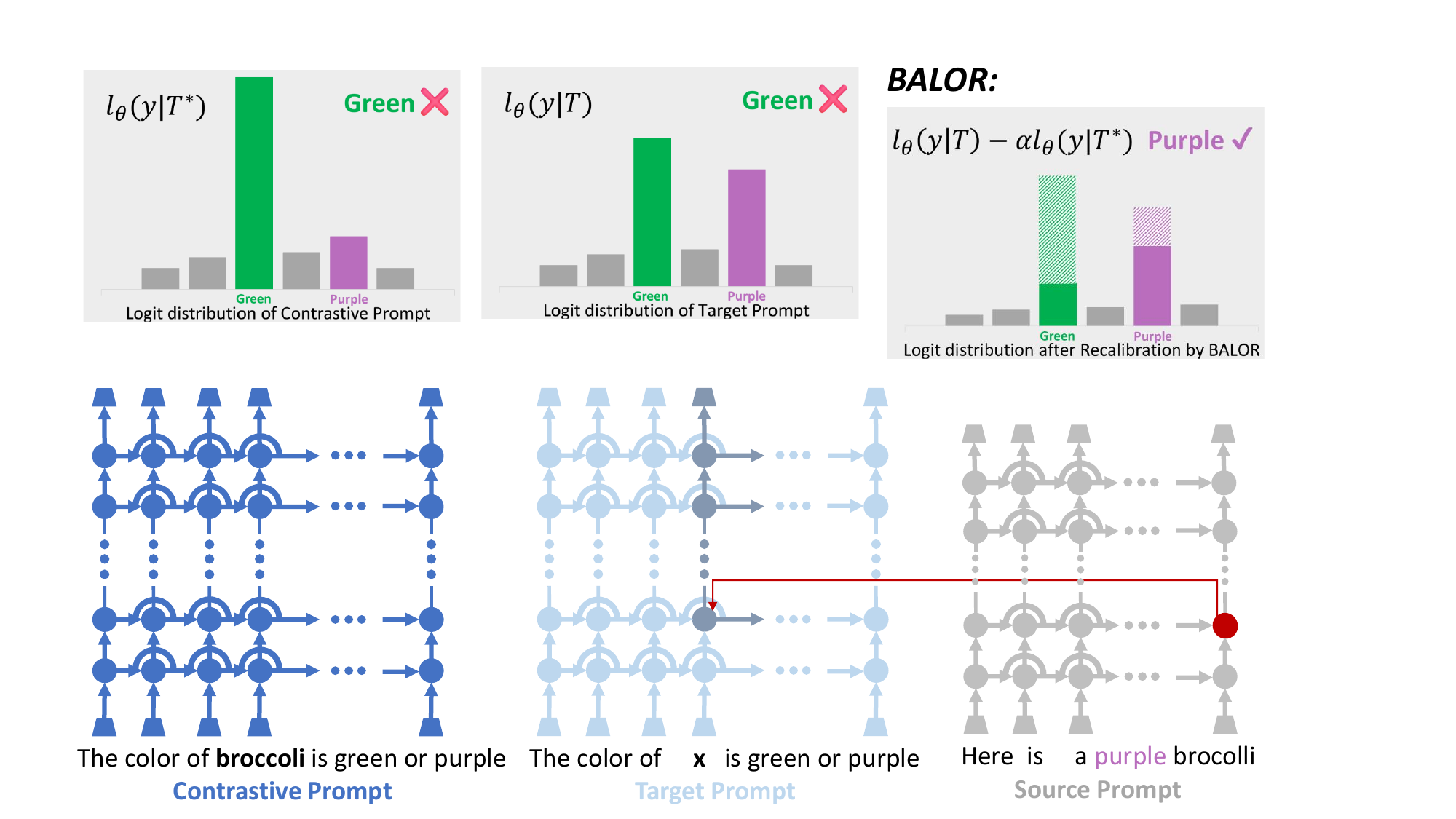}
    %\vspace{-6pt}
    \caption{An overview of BALOR method. We use the same example with Figure~\ref{fig:introexample}, while green is a biased attribute value and purple is a correct attribute value. With the recalibration approach between the logit distributions of contrastive prompt and target prompt in BALOR process, the LLM provides a correct response.}
    \label{fig:balor_method}
    %\vspace{-8pt}
\end{figure}

We proposed BALOR, a decoding method through logit recalibration. 
%A ``calibration'' step is applied to the models' output probabilities, which is adopted from the concept of contrastive decoding. 
In the BALOR method, we build a contrastive prompt pair $\langle T, T^* \rangle$. $T$ is the target prompt following the designation of~\citet{ghandeharioun2024patchscopesunifyingframeworkinspecting}, enabling an expressive decoding of a wide range of features. $T^*$ is the contrastive prompt, which replaces the placeholder token in $T$ with the original token in the source prompt that needs to be explained. Thus, the token in the contrastive prompt provided no contextual information, and the model can only explain the token with empirical knowledge and model bias. From $T^*$, we can obtain logits $l^*:=l_\theta(y\mid T^*)$ that only contain model priors. To emphasize the contribution of patched information, we aim to mitigate the undesired model bias highlighted by the contrastive prompt and allow the model to generate a response based on the remaining information when provided with the injected hidden representations. Thus, we recalibrate the logit distribution:
\[
p_{\text{balor}}(y\mid T)=\sigma\left[ (1+\alpha)l_\theta(y\mid T)-\alpha l_\theta(y\mid T^*)\right], 
\]
 $\sigma$ is the softmax function, and $\alpha$ is a hyperparameter that denotes the amplification level of the recalibration method.

% The contrastive pairs are:
% \begin{tcolorbox}
%     \textbf{Target Prompt ($T$) :}The \{task\} of \textcolor{blue}{\{$x$\} }is \{$a_{pri}$\} or \{$a_{sec}$\}?

%     \textbf{Contrastive Prompt ($T^*$) :} The \{task\} of \textcolor{blue}{\{noun\} }is \{$a_{ pri}$\} or \{$a_{sec}$\}?
% \end{tcolorbox}
% For the target prompt $T$, the placeholder $x$  will be replaced by the hidden representations of $noun$ with attribute information encoded, whereas no patching is applied to the contrastive prompt $T^*$. Since no extra information is provided, the model can only inference on empirical bias and knowledge. From $T^*$, we can obtain logits $l^*:=l_\theta(y\mid T^*)$ that only contain model priors. To emphasize the contribution of patched information, we aim to mitigate the undesired model bias highlighted by the contrastive prompt and allow the model to generate a response based on the remaining information when provided with the injected hidden representations. Thus, we recalibrate the logit distribution:
% \[
% p_{\text{balor}}(y\mid T)=\sigma\left[ (1+\alpha)l_\theta(y\mid T)-\alpha l_\theta(y\mid T^*)\right], 
% \]
%  $\sigma$ is the softmax function, and $\alpha$ is a hyperparameter that denotes the amplification level of the recalibration method. 

Figure~\ref{fig:balor_method} provides an intuitive illustration of how BALOR operates through logit recalibration with an example.
As shown in the left panel, when a contrastive prompt $T^*, \textit{``The color of broccoli is green or purple?''}$ is used without any contextual information, and the model exhibits a strong prior preference for the biased attribute value (\textit{green}). This preference reflects the model’s internal bias induced by imbalanced linguistic patterns. The middle panel shows that under the target prompt $T$, \textit{``The color of X is green or purple?''}, the contextual information is encoded in the patched hidden representation. Although the context-aligned attribute value (\emph{purple}) receives a higher logit than that in the contrastive prompt, its logit remains lower than that of the biased attribute value (\emph{green}). As a result, the response is still unfaithful. BALOR addresses this issue by isolating the bias component encoded in the unpatched prompt. As illustrated in the right panel, BALOR subtracts a scaled bias logit $l_\theta(y\mid T^*)$ from the target logit $l_\theta(y\mid T)$, yielding a recalibrated logit distribution. Consequently, the biased preference is suppressed, and the context-aligned attribute value (\emph{purple}) is amplified, as visualized by the dashed regions in the right panel. This logit-level adjustment effectively realigns the model’s decision with the injected contextual information without modifying model parameters or internal representations.

% We then provide a theoretical characterization of BALOR in the log-odds space. The following Theorem~\ref{thm:balor_logodds} explains why BALOR can consistently steer model predictions toward context-aligned tokens. It shows that as long as a candidate token is more supported by the patched target evidence than by the contrastive prior, BALOR will monotonically increase its relative preference as the scaling factor $\alpha$ grows. In the log-odds space, this effect is linear in $\alpha$, which guarantees that sufficiently large values of $\alpha$ will promote the context-consistent token to have the highest probability among competing candidates. The proof of this theorem is in Appendix~\ref{app:proof}. 

We then provide a theoretical characterization of BALOR in the log-odds space. Theorem~\ref{thm:balor_logodds} shows that if a candidate token is more supported by the patched token evidence than by the contrastive prior, BALOR monotonically increases its relative preference as the scaling factor $\alpha$ grows. This effect is linear in $\alpha$ in log-odds space, ensuring that sufficiently large $\alpha$ promotes the context-consistent token to the highest probability. The proof is given in Appendix~\ref{app:proof}.

\begin{theorem}[Log-odds contrastive amplification of BALOR]
\label{thm:balor_logodds}
Let \(l_\theta(y\mid T)\) and \(l_\theta(y\mid T^*)\) be the (pre-softmax) logits over vocabulary \(\mathcal{V}\).
Define BALOR as
\[
p_{\text{balor}}(y\mid T)
=\sigma\!\left((1+\alpha)\,l_\theta(y\mid T)-\alpha\,l_\theta(y\mid T^*)\right),\alpha\ge 0
\]
Then, for any two tokens \(y_1,y_2\in\mathcal{V}\), 
\begin{align*}
\log\frac{p_{\text{balor}}(y_1\mid T)}{p_{\text{balor}}(y_2\mid T)}
&=
\log\frac{p_\theta(y_1\mid T)}{p_\theta(y_2\mid T)}
\;+\;\\
&\quad
\alpha\left(
\log\frac{p_\theta(y_1\mid T)}{p_\theta(y_2\mid T)}
-
\log\frac{p_\theta(y_1\mid T^*)}{p_\theta(y_2\mid T^*)}
\right),
\end{align*}
where

\(p_\theta(\cdot\mid T)=\sigma(l_\theta(\cdot\mid T))\) and \(p_\theta(\cdot\mid T^*)=\sigma(l_\theta(\cdot\mid T^*))\).
Equivalently,
\begin{align*}
\log\frac{p_{\text{balor}}(y_1\mid T)}{p_{\text{balor}}(y_2\mid T)}
&=
(1+\alpha)\!\left(l_\theta(y_1\mid T)-l_\theta(y_2\mid T)\right)\\
&\quad -\alpha\!\left(l_\theta(y_1\mid T^*)-l_\theta(y_2\mid T^*)\right).
\end{align*}

In particular, if \(y_1\) is \emph{more supported by the patched target evidence than by the contrastive prior}, $
\log\frac{p_\theta(y_1\mid T)}{p_\theta(y_2\mid T)}
>
\log\frac{p_\theta(y_1\mid T^*)}{p_\theta(y_2\mid T^*)}$,
then the BALOR preference for \(y_1\) over \(y_2\) increases monotonically with \(\alpha\).
\end{theorem}

While Figure~\ref{fig:balor_method} depicts a single decoding step, BALOR is applied iteratively during autoregressive next-token prediction with the target and contrastive prompts decoding in parallel at each step. There are two modes of \methodname{}, \textbf{S}hared and \textbf{D}ivided. In mode S, the recalibrated distribution $p_{\text{balor}}(y \mid T)$ is used to sample the next token for both the target prompt and the contrastive prompt, keeping their generated sequences synchronized.
In mode D, $p_{\text{balor}}(y\mid T)$ is applied only to the target prompt for next-token prediction, while the contrastive prompt continues decoding with its original logits. This design allows BALOR to flexibly control whether bias correction is enforced jointly across both prompts or applied solely to the target prompt during generation.

%\di{your method havily depends on your dataset, how to generalize your method to all prompts?}

\section{Experiments}
This section introduces the experimental setup in Section~\ref{sec:main_expsetting}, and then discusses BALOR's overall performance in Section~\ref{sec:main_result}, and the ablation study in Section~\ref{sec:ablation_study}.

\subsection{Settings}\label{sec:main_expsetting}
\textbf{Datasets and Models} We use the biased part of the dataset across four tasks introduced in Section~\ref{sec:dataset}, fixing the sampling temperature to zero as the main result. In the ablation study, we use the biased part of the color task. In the sampling temperature ablation, we set Top-k=$50$ and Top-p=$0.9$. The models used in all experiments are those of Section~\ref{sec:casual_analysis}.

\textbf{Baselines} We compared \methodname{} with six baselines: Vanilla, which means the original Patchscopes without any debiasing; Closed-Book (CB), Internal-Evidence (IE), and Data-Bound (DB), three prompt-guidance approaches to guide LLMs only use contextual information~\cite{amplayo-etal-2023-query}; SFT~\cite{rafailov2023direct} and FairSteer~\cite{li2025fairsteerinferencetimedebiasing}. The implementation details of these baselines are provided in Appendix~\ref{app:baselines}.

\textbf{Prompts under Zero-shot and Few-shot Settings}
In both settings, the source prompt $S$ is \texttt{Here is an \{a$_{i, sec}$\} \{noun$_i$\}.} The target prompt $T$ for zero-shot setting is \texttt{The \{task\} of x is \{a$_{i, pri}$\} or \{a$_{i, sec}$\}?} The target prompt $T$ for $n$-shot setting is \texttt{The \{task\} of noun$_1$ is \{a$_{1, sec}$\}... the \{task\} of noun$_n$ is \{a$_{n, sec}$\}, the \{task\} of x is}, where $1...n$ are are distinct from $i$. Contrastive prompts $T^*$ for both settings are replaced \texttt{x} with \texttt{noun$_i$}. For the gender task, target prompts are modified to bypass gender alignment, details shown in Appendix~\ref{app:biased_dataset}.

\subsection{Main result analysis}\label{sec:main_result}
The results of the baseline methods under zero-shot and few-shot settings on biased datasets are reported in Table~\ref{tab:main_result_zeroshot} and Table~\ref{tab:main_result_fewshot}. Across all four biased tasks, \methodname{} achieves the best overall performance.

Overall, zero-shot settings tend to yield higher SR than few-shot settings, which is caused by task formulation. Zero-shot settings are framed as binary-choice questions with a constrained output space, whereas few-shot settings require free-form generation conditioned on in-context examples, which substantially expands the space of possible attribute values. This effect is more pronounced for tasks with a large attribute value space, such as color and culture, where open-ended generation introduces greater linguistic variability. In contrast, for tasks with relatively constrained attribute spaces, such as gender and age, the performance gap between zero-shot and few-shot settings is less evident.

The three prompt-based methods, CB, IE, and DB, can sometimes improve the SR compared to vanilla Patchscopes. In particular, IE achieves the highest SR on the gender task for Llama2-7b and Llama3-8b under few-shot settings. However, their performance is highly unstable, and they frequently underperform the vanilla baseline. Moreover, determining an optimal prompt is difficult, as different prompts perform better across different tasks, resulting in inconsistent outcomes. Similarly, SFT methods do not consistently outperform vanilla Patchscopes. Their performance degrades substantially when the biased dataset used for fine-tuning conflicts with the evaluation data, often leading to much lower SR than the vanilla baseline. This behavior highlights an inherent limitation of SFT debiasing approaches. FairSteer exhibits a similar limitation, which becomes more pronounced when the debiasing direction is inaccurate. In the gender task under zero-shot settings, FairSteer performs worse than the vanilla baseline across all four models. 

In contrast, \methodname{} performs debiasing at the decoding stage by modifying the logit distribution, thereby avoiding the limitations of prompt-based and fine-tuning-based methods. Under both zero-shot and few-shot settings, \methodname{} consistently outperforms vanilla Patchscopes and achieves better performance than all six baselines in nearly all cases. These results demonstrate that \methodname{} attains both strong performance and high consistency, validating its effectiveness as a more faithful mechanism for hidden representation explanation.

However, we observe that mode D is less stable than mode S. While mode S consistently outperforms all six baselines, mode D occasionally yields a lower SR. This instability arises because the target and contrastive prompts in mode D follow different next-token prediction strategies, causing the two input sequences to diverge progressively. As a result, the resulting logit distributions reflect not only differences in the hidden representations, but also discrepancies introduced by newly generated tokens. This divergence reduces the stability of Mode D, although it can sometimes lead to higher SR.

% However, we also found that mode D is not as stable as mode S, while mode S consistently outperforms all six baselines while mode D sometimes get a lower SR.
% This is because the target and contrastive prompt use different next token prediction strategies, the two input sequences diverge more and more, causing their new logit distributions to reflect not only differences in the hidden representations we aim to explain, but also differences introduced by the newly generated tokens. This divergence reduces the stability of mode D, but sometimes helps it to achieve higher SR.

\begin{table*}[h!]
\centering
\resizebox{0.95\linewidth}{!}{
\small
\begin{tabular}{lc|cccccccc}
\toprule
Task & Model & Vanilla& CB&IE&DB &Fine-tuning & FairSteer & \methodname{} (S) & \methodname{} (D) \\
\midrule

\multirow{4}{*}{Color} 
& Llama2-7b    & 37.05\% & 48.70\% & \underline{49.09\%} & 47.59\% &40.36\%& 40.36\% & \textbf{49.39\%} & 49.39\% \\
& Llama3.2-1b  & 39.10\% & 43.17\% & 58.58\% & \underline{62.94\%} &21.80\%& 38.81\% & \textbf{71.66\%} & 48.84\% \\
& Llama3-8b    & 23.40\% & 32.40\% & 24.43\% & 24.32\% &\underline{43.48\%}& 32.29\% & \textbf{56.42\%} & 39.23\% \\
& Qwen3-4b     & 48.20\% & 50.79\% & 46.61\% & 46.21\% &51.70\%& 51.49\% & \underline{73.45\%} & \textbf{74.80\%} \\
\midrule

\multirow{4}{*}{Gender} 
& Llama2-7b & 24.06\%&36.04\% &\underline{41.70\%}&37.92\%&41.02\%& 22.64\% & \textbf{46.42\%} &43.70\%\\
& Llama3.2-1b & 3.33\%&46.67\%&9.99\%& 26.64\%&3.33\%&0.00\% &\underline{59.94\%} &\textbf{63.27\%}\\
& Llama3-8b & 0.00\%&50.00\%&50.00\%&47.56\% &32.93\%&0.00\% &\underline{54.39\%} &\textbf{51.71\%}\\
& Qwen3-4b & 34.78\%&2.17\%&8.70\%&2.17\% &48.58\%&31.13\% &\textbf{76.09\%} &\underline{58.26\%}\\
\midrule

\multirow{4}{*}{Culture} 
& Llama2-7b & 36.84\%&37.72\%&43.49\%& 42.11\% &38.60\%& 17.54\% & \textbf{54.39\%} & \underline{45.61\%}\\
& Llama3.2-1b & 35.87\%&29.35\%&35.87\%&39.13\%&4.35\% &45.65\%& \textbf{57.61\%} & \underline{40.43\%}\\
& Llama3-8b & 9.68\% &8.87\%&10.48\%&22.58\%&18.55\%&6.45\%& \textbf{39.52\%}  & \underline{37.90\%}\\
& Qwen3-4b & 45.69\% &43.96\%&43.73\%&47.41\%&54.03\%&50.09\%& \textbf{74.14\%} & \underline{70.16\%}\\
\midrule

\multirow{4}{*}{Age} 
& Llama2-7b & 50.00\%& 46.34\% &58.54\%&\underline{60.98\%}&58.54\%&43.90\% & \textbf{69.51\%} &67.07\% \\
& Llama3.2-1b & 50.00\% &50.00\%&47.22\%&50.00\%&47.78\%&41.46\%& \underline{50.00\%} & \textbf{91.67\%}\\
& Llama3-8b & 67.24\% &50.00\%&39.66\%&50.00\%&65.56\%&44.83\%& \underline{70.69\%} & \textbf{72.41\%}\\
& Qwen3-4b & 56.67\% &56.10\%&50.00\%&35.37\%&70.76\%&56.10\%& \underline{78.05\%} & \textbf{98.78\%}\\

\bottomrule
\end{tabular}}
\caption{Compare \methodname{} with baselines under zero-shot settings. \textbf{Bold} and \underline{underline} denote the highest and second-highest SR. }
\label{tab:main_result_zeroshot}
\end{table*}

% \methodname{} (D) is not as stable as mode S, since if the target and contrastive prompt use different next token prediction stragegy, he two input sequences diverge more and more, causing their new logit distributions to reflect not only differences in the hidden representations we aim to explain, but also differences introduced by the newly generated content. This divergence reduces the stability of the method. Nevertheless, it still outperforms the baseline across all four models.

\begin{table*}[h!]
\centering
\resizebox{0.95\linewidth}{!}{
\small
\begin{tabular}{lc|cccccccc}
\toprule
Task & Model & Vanilla& CB&IE&DB &Fine-tuning & FairSteer & \methodname{} (S) & \methodname{} (D) \\
\midrule

\multirow{4}{*}{Color} 
& Llama2-7b    & 5.42\% & 7.83\% & 3.01\% & 1.51\% &18.37\%& 4.82\% & \textbf{52.10\%} & \underline{48.19\%} \\
& Llama3.2-1b  & 14.24\% & 21.07\% & 8.14\% & 15.12\% &3.78\%& 19.77\% & \textbf{69.77\%} & \underline{65.99\%} \\
& Llama3-8b    & 10.97\% & 6.63\% &8.70\% & 8.60\% &41.72\%& 14.50\% & \textbf{47.62\%} & \underline{43.47\%} \\
& Qwen3-4b     & 18.41\% & 13.66\% & 14.44\% & 15.50\% &29.17\%& 22.09\% & \textbf{68.90\%} & \underline{67.54\%} \\
\midrule

\multirow{4}{*}{Gender} 
& Llama2-7b & 34.91\%&69.34\%& \textbf{82.08\%}& \underline{80.19\%}&36.23\%& 36.80\% & 37.74\% &40.57\%\\
& Llama3.2-1b & 6.67\%&16.67\%&10.00\%& 13.33\%&13.33\%&10.00\% & \underline{26.67\%} & \textbf{36.67\%}\\
& Llama3-8b & 12.20\%&28.05\%& \textbf{34.15\%}&17.07\% &1.22\%&21.95\% & \underline{29.27\%} &28.05\%\\
& Qwen3-4b & 26.08\%&17.39\%&15.21\%&21.74\% &17.39\%&21.74\% & \textbf{41.30\%} & \underline{39.13\%}\\
\midrule

\multirow{4}{*}{Culture} 
& Llama2-7b & 10.53\%&14.03\%&9.65\%& 14.91\% &8.77\%& 8.77\% & \underline{55.26\%} & \textbf{59.64\%}\\
& Llama3.2-1b & 0.00\%& 2.17\%& 9.78\%&  10.87\%&   8.70\% &13.04\%& \textbf{40.65\%} & \underline{40.21\%}\\
& Llama3-8b & 14.52\% & 4.03\%& 8.06\%& 4.03\%& 1.61\%& 12.90\%&  \textbf{37.09\%}&  \underline{35.32\%}\\
& Qwen3-4b & 8.62\%& 4.31\%& 9.48\%& 8.62\%& 27.59\%& 12.07\%& \textbf{62.07\%}& \underline{57.76\%}\\
\midrule

\multirow{4}{*}{Age} 
& Llama2-7b & 56.10\%&  42.68\% & 48.78\%& 39.02\%& 59.76\%&51.22\% &  \textbf{80.48\%} & \underline{74.39\%} \\
& Llama3.2-1b & 0.00\% & 19.44\%& 13.89\%& 2.78\%& 16.67\%& \underline{55.56\%}&  \textbf{63.89\%} &  58.33\%\\
& Llama3-8b & 17.24\% & 31.03\%& 27.58\%& 34.48\%&10.34\%& \underline{39.93\%}&  \textbf{89.65\%}  &  84.48\% \\
& Qwen3-4b &  39.02\% & 40.24\%& 51.21\%& 60.97\%&41.46\%&46.34\%&  \textbf{86.58\%} & \underline{79.26\%}\\

\bottomrule
\end{tabular}}
\caption{Compare \methodname{} with baselines under few-shot settings. \textbf{Bold} and \underline{underline} denote the highest and second-highest SR. }
\label{tab:main_result_fewshot}
\end{table*}
\subsection{Ablation Study}\label{sec:ablation_study}
\textbf{Hyperparemeter $\alpha$.}
%We evaluate the effect of the hyperparameter $\alpha$ on the performance of \methodname{}. 
As shown in Figure~\ref{fig:ablation_alpha}, \methodname{} consistently yields positive SR gains over the vanilla Patchscopes baseline across all four models for a wide range of $\alpha$ values. The Shared mode achieves stable improvements throughout $\alpha \in [0.4, 4.0]$, achieving substantial peak gains of 12–33\%, while maintaining positive performance as $\alpha$ increases. The Divided mode exhibits more variability and smaller gains. Although the relative gains are negative on $\alpha=2.8$ for Llama2-7b and $\alpha \ge 0.8$, but still improve over the vanilla baseline for all other $\alpha$ settings on four models. While mode S is more stable, we recommend using this Shared mode in the application. 

%for example, +33.2\% for Llama3-8b and +25.2\% for Qwen3-4b at $\alpha=0.8$.
%It also surpasses prompt-guidance methods in most settings, indicating that direct latent steering provides more stable and interpretable control. Performance remains robust for $\alpha \in [0.4, 4.0]$ with step size 0.4, suggesting low sensitivity to recalibration strength.
%Moreover, \methodname{} also outperforms prompt-guidance methods in most configurations, indicating that direct latent steering provides more stable and interpretable control than prompt-guidance interventions. The performance remains robust across $\alpha$ values from 0.4 to 4.0 with step width 0.4, suggesting that \methodname{} (S) is not overly sensitive to the recalibration strength. 

\begin{figure}[h!]
    \centering
    \includegraphics[width=0.95\linewidth,
    trim=0 10 0 5,
        clip]{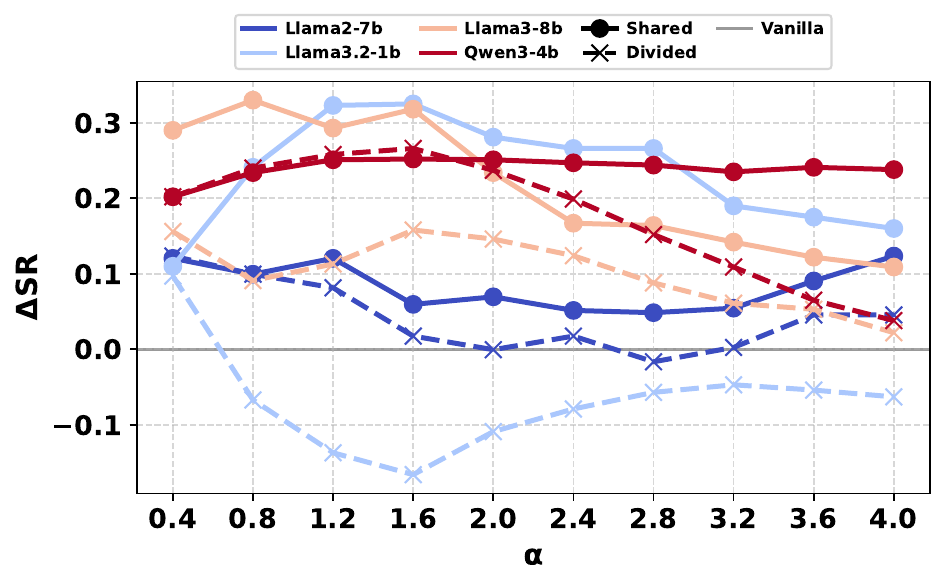}
    \caption{Relation SR gain by hyberparameter $\alpha$ of four models compared with vanilla Patchscopes.
    }
    \label{fig:ablation_alpha}
\end{figure}

\textbf{Sampling Temperature.} We further evaluate \methodname{} under varying sampling temperatures, which control inference randomness, ranging from 0.0 to 1.6 for all LLMs (Figure~\ref{fig:temperature}). Sampling temperature has little impact on the SR of mode S, whereas mode D exhibits a slight upward trend. This indicates that mode S remains stable across temperatures, while mode D benefits from higher temperatures, as divergence between prompt pairs means the optimal token may not always have the highest post-recalibration probability. For statistical analysis, we apply the Kruskal-Wallis test~\cite{kruskal1952use} to the SR distributions. The definition and result of the Kruskal-Wallis test are provided in Appendix~\ref{app:kruskal}.
\begin{figure}[h!]
    \centering
    \includegraphics[width=\linewidth,
    trim=0 0 0 20,
        clip]{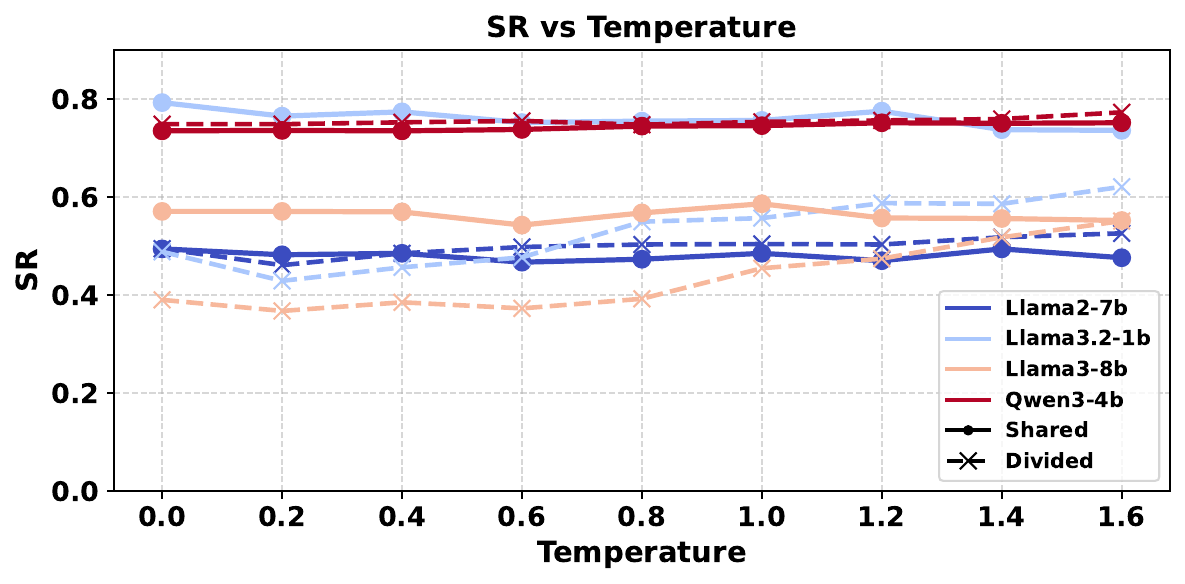}
    \caption{SR by sampling temperature of four models for \methodname{}.
    }
    \label{fig:temperature}
\end{figure}

\section{Related Work}
\textbf{LLM Interpretability.} Since LLMs are widely used in a variety of areas, it is important to ensure their explainability. Traditional interpretability methods, such as probing classifiers ~\citep{alain2018understandingintermediatelayersusing, belinkov-2022-probing,DBLP:journals/corr/abs-2508-10599,DBLP:journals/corr/abs-2506-15617,yang2025exploringpersonalitytraitsllms}, aim to measure whether a particular linguistic or semantic property can be linearly predicted from intermediate activations.
%Although these methods provide quantitative insight into representational content, they often fail to explain whether the information is faithfully used during inference.
More recent studies have explored interpretability by translating hidden representations into natural language.~\citet{belrose2025elicitinglatentpredictionstransformers} use tuned lens to reconstruct the semantics of hidden vectors by projecting them back into interpretable token spaces. Building upon this direction, the Patchscopes framework~\cite{ghandeharioun2024patchscopesunifyingframeworkinspecting} introduces a modular approach to interpreting hidden representations by decoding them through a model’s generative process.

\textbf{Debiasing LLMs. }The LLMs exhibit modality preferences and inherent biases~\cite{shrawgi-etal-2024-uncovering, gallegos-etal-2024-bias, zhang2025modalitiesconflictunimodalreasoning}, motivating the development of debiasing methods. Debiasing methods can be grouped into two main categories: 
(1) Fine-tuning-based methods, which mitigate bias in LLMs by retraining the model on a well-constructed dataset~\cite{chen2023fastmodeldebiasmachine}. The retraining methods such as contrastive learning~\cite{zhou-etal-2022-debiased, dong-etal-2023-co2pt, wu2025fisher} and reinforcement learning~\cite{allam-2024-biasdpo, qureshi2024refine} and are usually combined with a parameter-efficient-fine-tuning technique such as LoRA~\cite{hu2022lora} and BitFit~\cite{zaken2022bitfit}. 
(2) Inference-time intervention methods intervene in the model at the inference stage. Prompt-based intervention methods use carefully designed prompts to lead LLMs toward unbiased outputs during generation~\cite{dong-etal-2023-co2pt, oba-etal-2024-contextual, gallegos2024selfdebiasinglargelanguagemodels, li2025fairsteerinferencetimedebiasing}. Other methods, such as Fairsteer~\cite{li2025fairsteerinferencetimedebiasing}, learn a steering vector and then use it to adjust biased activation during inference.

\section{Conclusion}
In this work, we show that imbalanced linguistic patterns will lead to model bias in LLMs, and this kind of model bias makes the interpretation tool Patchscopes unfaithful. We built a dataset to investigate this, and through controlled experiments, we show that the model bias induced by imbalanced linguistic patterns overrides contextual information, which leads to unfaithfulness. Building on this insight, we propose \methodname{}, a contrastive logit-recalibration strategy to mitigate this phenomenon, enhancing the faithfulness of Patchscopes.

\section*{Impact Statement}
This paper aims to improve the reliability of LLM interpretability by analyzing and mitigating unfaithfulness in the Patchscopes framework. We show that imbalanced linguistic patterns can induce model bias that distorts explanation faithfulness, and we propose an inference-time decoding strategy to reduce this effect without modifying model parameters. By improving the faithfulness of Patchscopes, our approach may help researchers better understand how LLMs encode and use contextual information, which is important for trustworthy deployment.
We do not foresee negative societal consequences arising directly from this work. However, as with all interpretability tools, explanations produced using our method should not be treated as definitive evidence of model reasoning, and they should be used alongside complementary evaluation and auditing techniques. Overall, we believe the benefits of improving explanation faithfulness outweigh potential risks.

\bibliography{ref}
\bibliographystyle{icml2026}

\appendix

\section{Casual Analysis on Other Models} \label{app:casual}

In this section, we provide the causal analysis results on all four models we used. Across all models, the logistic regression results show a consistent positive association between the frequency gap $\Delta f$(defined in Section~\ref{sec:casual_analysis}) and the likelihood of color bias, As summarized in Table \ref{table:casual_analysis}, one standard deviation increase in $\Delta f$ leads to a substantial rise in the odds of exhibiting bias, with odds ratios (OR) ranging from 5.726 (Llama2-7B) to 9.338 (Llama3.2-1B). The corresponding logit coefficients are all positive and statistically significant, confirming that training frequency imbalance meaningfully predicts model bias. Furthermore, the ROC-AUC values between 0.579 and 0.619 indicate moderate but consistent discriminative power across models.

\begin{table}[h!]
\centering
\resizebox{\linewidth}{!}{
\begin{tabular}{lccc}
\toprule
\textbf{Model} & \textbf{Metric} & \textbf{Mean Value} & \textbf{95\% CI} \\
\midrule
\multirow{3}{*}{Llama2-7b} 
 & Logit coeff (+1 SD) & 1.745 & (0.866, 2.624) \\
 & OR (+1 SD)          & 5.726 & (2.376, 13.796) \\
 & ROC-AUC             & 0.606 & (0.549, 0.659) \\
\midrule
\multirow{3}{*}{Llama3.2-1b} 
 & Logit coeff (+1 SD) & 2.234 & (1.487, 2.981) \\
 & OR (+1 SD)          & 9.338 & (4.425, 19.706) \\
 & ROC-AUC             & 0.619 & (0.580, 0.655) \\
\midrule
\multirow{3}{*}{Llama3-8b} 
 & Logit coeff (+1 SD) & 1.894 & (1.295, 2.493) \\
 & OR (+1 SD)          & 6.645 & (3.651, 12.097) \\
 & ROC-AUC             & 0.579 & (0.546, 0.609) \\
\midrule
\multirow{3}{*}{Qwen3-4b} 
 & Logit coeff (+1 SD) & 1.866 & (1.312, 2.420) \\
 & OR (+1 SD)          & 6.464 & (3.715, 11.247) \\
 & ROC-AUC             & 0.599 & (0.569, 0.629) \\
\bottomrule
\end{tabular}}
\caption{Repeated undersampling results for each model.}
\label{table:casual_analysis}
\end{table}

% \section{Details about source and target prompt settings}\label{app:prompt_detail}

% For simple implementation, the placeholder symbol $x$ can be repeated to match the token length of the noun. To mitigate option position bias, we additionally use a counterpart target prompt with the attribute options swapped and report performance averaged over both prompts.

\section{Details about Locating the Bias-sensitive Layer}
\subsection{Logit Difference}\label{app:logit_diff}
To verify that a given layer indeed reflects the model’s understanding of the input, we append an interrogative clause ``What \{attribute\} is \{noun\}?'' to the source prompt. This addition explicitly forces the model to process the attribute associated with the noun, thereby embedding the desired semantic information within the evolving hidden representations.

After injecting this clause, we use logit-lens to measure how the model internally differentiates the secondary attribute, which is shown less in linguistic pattern but provided in contextual information, from the primary attribute, which is the biased attribute from linguistic pattern at layer $l$. We examine the logit values difference of the secondary and primary attributes, where the logit value is computed by mapping the hidden representation at layer $l$ to the vocabulary space using the unembedding matrix $W_o$. This reflects how strongly the layer-level hidden representation favors the secondary attribute over the primary one. Intuitively, if a layer meaningfully encodes the color information carried by the patched representation, the logit associated with the secondary color should be higher than that of the primary color.

\subsection{Linear Probing}\label{app:linear_probing}
We train a multi-class linear probe that can classify the noun tokens into different colors on a general dataset. $W \in \mathbb{R}^{C\times d}$ is the probe's weight matrix for $C$ color classes when $d$ is the dimension of hidden representations. After training, each row vector $W_c$ corresponds to a learned attribute direction of color $c$~\cite{kim2018interpretability}. Thus, this attribute direction $w$ can be directly taken as $w=W_c\in \mathbb{R}^d$.

\subsection{Correlation Analysis}
\subsubsection{Spearman Correlation}\label{app:spearman}
The Spearman correlation coefficient $\rho$ is a non-parametric measure of rank correlation that evaluates the strength and direction of a monotonic relationship between two variables~\cite{spearman1961proof}.  
Given paired observations $(x_i, y_i)$ for $i = 1, \ldots, n$, we first convert them to ranks $R(x_i)$ and $R(y_i)$, and then compute:
\[
\rho = 
\frac{\sum_{i=1}^n (R(x_i) - \overline{R_x})(R(y_i) - \overline{R_y})}
{\sqrt{\sum_{i=1}^n (R(x_i) - \overline{R_x})^2} 
\sqrt{\sum_{i=1}^n (R(y_i) - \overline{R_y})^2}},
\]
where $\overline{R_x} = \frac{1}{n}\sum_{i=1}^{n}R(x_i)$ and $\overline{R_y} = \frac{1}{n}\sum_{i=1}^{n}R(y_i)$.  
Alternatively, when there are no tied ranks, it can be expressed as
\[
\rho = 1 - \frac{6\sum_{i=1}^{n} d_i^2}{n(n^2 - 1)}, \quad d_i = R(x_i) - R(y_i).
\]
The coefficient ranges from $-1$ (perfect negative correlation) to $+1$ (perfect positive correlation), 
with $\rho = 0$ indicating no monotonic association.

\subsubsection{Isotonic Regression}\label{app:Isotonic}
Isotonic regression~\cite{miles1959complete} is a non-parametric technique for fitting a monotonic (non-decreasing) function to a set of paired observations $(x_i, y_i)$. It estimates values $\hat{f_i}$ by solving
\[
\min_{f_1 \le f_2 \le \cdots \le f_n} \sum_{i=1}^{n} w_i (y_i - f_i)^2,
\]
subject to the ordering constraints $f_i \le f_{i+1}$, where $w_i$ are optional non-negative weights.  
This problem is efficiently solved using the Pool Adjacent Violators Algorithm (PAVA), which produces a piecewise-constant fit that enforces the desired monotonic relationship.

\subsubsection{Correlation Analysis Results}\label{app:correlation}
The results in Table \ref{tab:spearman_correlation} demonstrate that these two metrics, LD and GSA, do not have a monotonic relationship since $\rho$ is close to zero with a small $p$-value. We further characterize the correlation of LD and GSA with isotonic regression, and the fitted curves are shown in Figure \ref{fig:correlation_all}. These results further demonstrate that these two metrics are independent.
\begin{table}[t]
\centering

\small
\begin{tabular}{lcc}
\toprule
Model & Spearman $\rho$ & $p$-value \\
\midrule
Llama2-7B   & $0.010$  & $2.491\times10^{-1}$ \\
Llama3.2-1B & $-0.147$ & $4.852\times10^{-33}$ \\
Llama3-8B   & $-0.006$ & $5.023\times10^{-1}$ \\
Qwen3-4B    & $-0.040$ & $9.281\times10^{-6}$ \\
\bottomrule
\end{tabular}
\caption{Spearman correlation between Logit Difference and Gradient Similarity Alignment across models.}
\label{tab:spearman_correlation}
\end{table}

\begin{figure}[t]
    \centering
    \includegraphics[width=\linewidth]{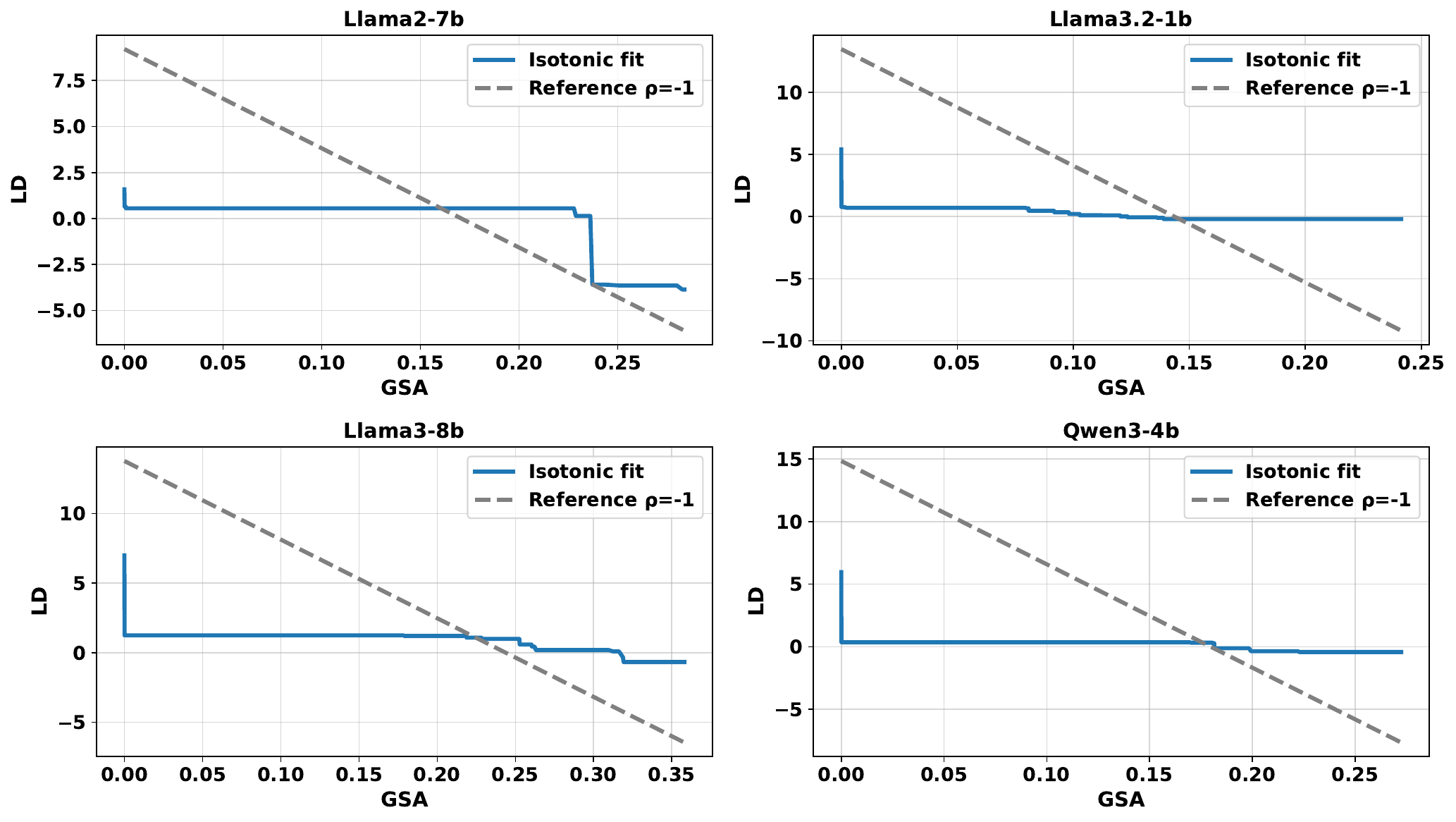}
    \caption{Isotonic regression results across models. The gray line is the reference curve while the blue line is the fitted curve.}
    \label{fig:correlation_all}
\end{figure}

\section{Proof of Theorem}\label{app:proof}
\textit{Proof.}

%By definition,
\(
p_{\text{balor}}(y\mid T)\propto
\exp\left((1+\alpha)l_\theta(y\mid T)-\alpha l_\theta(y\mid T^*)\right).
\)
Therefore, for any \(y_1,y_2\),
\begin{align*}
\log\frac{p_{\text{balor}}(y_1\mid T)}{p_{\text{balor}}(y_2\mid T)}
&=
(1+\alpha)\!\left(l_\theta(y_1\mid T)-l_\theta(y_2\mid T)\right)\\
&\quad -\alpha\!\left(l_\theta(y_1\mid T^*)-l_\theta(y_2\mid T^*)\right),
\end{align*}
which gives the second equality. For the first equality, note that softmax log-odds satisfy
\(
\log\frac{\sigma(l)(y_1)}{\sigma(l)(y_2)}=l(y_1)-l(y_2)
\),
so substituting \(p_\theta(\cdot\mid T)=\sigma(l_\theta(\cdot\mid T))\) and similarly for \(T^*\) yields the stated decomposition.
Finally, the monotonicity follows because the right-hand side is affine in \(\alpha\) with slope
\(
\log\frac{p_\theta(y_1\mid T)}{p_\theta(y_2\mid T)}-
\log\frac{p_\theta(y_1\mid T^*)}{p_\theta(y_2\mid T^*)}
\),
which is positive under the condition.
%\end{proof}

\section{Baselines}\label{app:baselines}
\subsection{Prompting}

% \subsection{Can model bias Be Mitigated?}
% We first employ a prompt-guidance intervention method to explore whether the model bias can be controlled. 
Prompt-guidance intervention approach is the easiest to implement way for model debiasing. We design three distinct guidance prompts that serve as prefix instructions to guide the model in generating non-biased responses based on provided contextual information only~\cite{amplayo-etal-2023-query}.

\begin{tcolorbox}
\textbf{Closed-Book}: Do not use outside or empirical knowledge, directly answer. \textcolor{blue}{\textless target prompt \textgreater}.

\textbf{Internal-Evidence}: Use only internal evidence from the input, exclude background knowledge. \textcolor{blue}{\textless target prompt \textgreater}.

\textbf{Data-Bound}: Answer strictly from the given data, avoid any external reasoning. \textcolor{blue}{\textless target prompt \textgreater}.
\end{tcolorbox}

% As shown in Figure~\ref{fig:prompt_result}, prompts-guidance cases generally outperform the baseline, although in Qwen3-4b, the guidance of \textit{prompt2} and \textit{prompt3} has a slightly lower accuracy. These result indicates that prompts can steer the model to inference based on the information encoded in patched hidden representations, instead of relying on empirical bias and knowledge. This tendency means that model bias can be mitigated by inference-time-steering.

% \begin{figure}[t]
%     \centering
%     \includegraphics[width=\linewidth]{figure/prompt.png}
%     \caption{The result of prompt-guidance of each model.
%     }
%     \label{fig:prompt_result}
% \end{figure}

\subsection{Fine-tuning}
In order to prevent the effect of model bias, fine-tuning is a well-known debiasing method. We use Direct Preference Optimization (DPO)~\cite{rafailov2023direct} combined with Low-Rank adaptation (LoRA)~\cite{hu2022lora} here to perform lightweight preference fine-tuning on Llama2-7b, Llama3.2-1b, Llama3-8b, and Qwen3-4b. To reduce model bias, we adopt 1,507
pairs of biased and unbiased data from~\cite{nangia-etal-2020-crows}, of which 95\% are used for training and 5\% are used for validation. This encourages the model to prefer more fairness and faithful responses.

\subsection{FairSteer}
FairSteer is a vector-steering intervention method. Following \cite{li2025fairsteerinferencetimedebiasing}, we first trained a steering vector to capture the overall direction from biased case to unbiased case. During the inference stage, we modify the midlayer activation before the patched layer along the steering vector's direction by adding the steering vector on the activation.

\section{Ablation Study of Hyper parameter $w$}\label{app:parameter}

The ablation study of hyper parameter $w$ is shown in Figure \ref{fig:hyper_parameter_w_all}. The best layer under each $w$ is highlighted in the figure. Across all $w$ values, there are only 2 layers as candidates for Llama2-7b, and 3 layers as candidates for Llama3.2-1b, Llama3-8b and Qwen3-4b. As stated in \cite{lepori-etal-2025-racing} and \cite{ghandeharioun2024patchscopesunifyingframeworkinspecting}, biases or misinterpretations often arise in the middle layers of transformer models. Thus, we set $w=0.8$ for all models and select layer 13 for Llama2-7b, layer 11 for Llama3.2-1b, layer 21 for Llama3-8b and layer 24 for Qwen3-4b.

\begin{figure}[t]
    \centering
    \includegraphics[width=\linewidth]{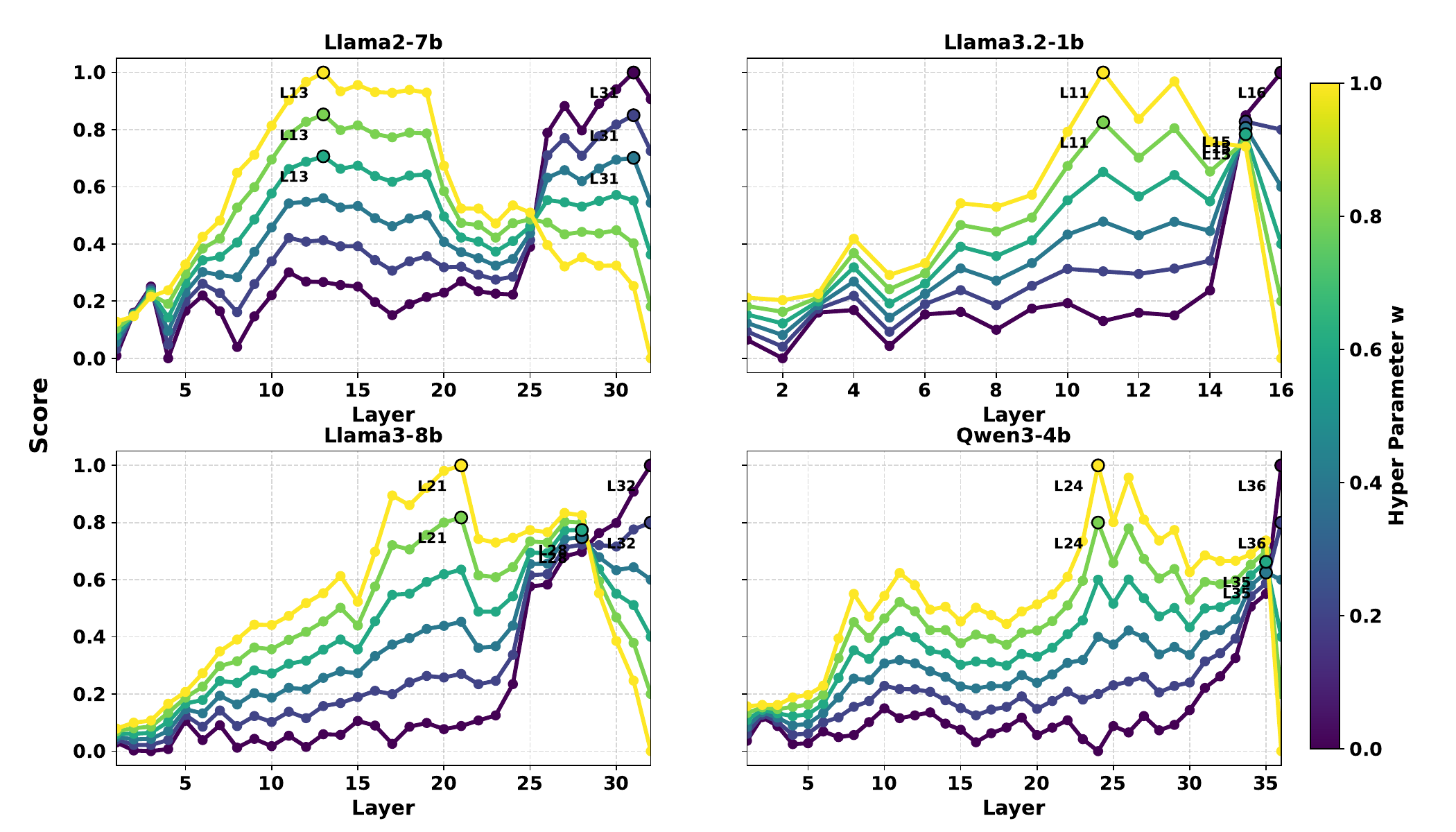}
    \caption{Layer-wise score progression across different models under varying hyperparameter $w$. Each curve represents a different $w$ value and the color of the curve regarding the value of $w$ according to the color bar on the right.}
    \label{fig:hyper_parameter_w_all}
\end{figure}

\section{Kruskal-Wallis test}\label{app:kruskal}
\subsection{Defination}
The Kruskal-Wallis test~\cite{kruskal1952use} is a non-parametric statistical test used to determine 
whether multiple independent groups originate from the same distribution.  
It generalizes the Wilcoxon rank-sum test~\cite{wilcoxon1992individual} to the case of $k$ groups.

Suppose we have $k$ independent groups $G_1, G_2, \ldots, G_k$ with group sizes $n_1, n_2, \ldots, n_k$, and a total sample size 
$
N = \sum_{j=1}^{k} n_j.$
At the ranking step, all observations from all groups are pooled together and assigned ranks 
$R_{ij}$, where $R_{ij}$ is the rank of the $i$-th observation in group $j$ 
after sorting the entire dataset.  
The sum of ranks for each group is$
T_j = \sum_{i=1}^{n_j} R_{ij}$.

Then the Kruskal--Wallis statistic $H$ is computed as
\[
H = 
\frac{12}{N(N+1)}
\sum_{j=1}^{k} \frac{T_j^2}{n_j}
- 3(N+1).
\]

When tied ranks exist, a tie-adjustment factor is applied to ensure validity.  
Under the null hypothesis
\[
H_0: \text{all } k \text{ groups come from the same distribution}.
\]
The statistic $H$ follows an asymptotic chi-square distribution $
H \sim \chi^2_{k-1}$.

The decision rule is that given the computed statistic $H$ and degrees of freedom $k - 1$, the $p$-value is
\[
p = 1 - F_{\chi^2_{k-1}}(H),
\]
where $F$ is the chi-square CDF.  
A small $p$-value ($<0.05$) indicates that at least one group differs in distribution 
from the others.

\subsection{Kruskal-Wallis test Result}
 The result is shown in Table \ref{tab:kruskal_test} and for all LLMs, $p$-value $>0.05$ indicates that the result is non-significant and thus, sampling temperature does not introduce measurable distributional shifts in the SR of \methodname{} (S). For \methodname{} (D), while the SR shows a slight increase, the distribution also shifts during the sampling temperature increase. This result also supports the analysis presented in Section \ref{sec:ablation_study}, regarding the performance differences between modes S and D, as well as the trend of SR varying with sampling temperature.
\begin{table}[h!]
\centering
\begin{tabular}{c|cc}
\toprule

Model & H(9) & p-value \\
\midrule
\multicolumn{3}{c}{Mode S} \\
\midrule
Llama2-7b      & 0.6584  & 0.9953 \\
Llama3.2-1b    & 6.2894  & 0.3916 \\
Llama3-8b      & 3.1877  & 0.7850 \\
Qwen3-4b       & 1.8333  & 0.9344 \\
\midrule
\multicolumn{3}{c}{Mode D} \\
\midrule
Llama2-7b      & 7.6408  & 0.3653 \\
Llama3.2-1b    & 113.6126 & 1.63$\times$10$^{-21}$ \\
Llama3-8b      & 140.8080 & 3.44$\times$10$^{-27}$ \\
Qwen3-4b       & 3.4114  & 0.8445 \\
\bottomrule
\end{tabular}
\caption{Kruskal--Wallis test results by model for \methodname{} steering.}
\label{tab:kruskal_test}
\end{table}

\section{Dataset}
\subsection{Color Task}\label{app:color_dataset}
\subsubsection{Make datapoints for Color Task}
The color task contains 1,146 product nouns and uses their related colors as attributes. The nouns are selected from the LVIS~\citep{gupta2019lvis} dataset. We only use the category names, which are 1,202 nouns in total. The colors related to these nouns are then extracted from the Redpajiama subset~\citep{weber2024redpajama}. Based on the co-occurrence frequency of nouns and colors in the RedPajama dataset, we matched each noun with two colors: the most frequently co-occurring color (primiary color) and the second most frequently one (secondary color). Some nouns basically do not have color (e.g Internet, person, chap), or we only get ``color not found''. We manually removed 56 such data points and leaving 1,146 nouns have valid primary and secondary color.

\subsubsection{Select Datapoints in the Color Task}
We divide the color datapoints by conducting an \textit{option-swapping QA task} to find the subset where the model shows bias toward primary color. Since LLMs have bias in multipul-choice task \citep{zheng2024largelanguagemodelsrobust,pezeshkpour-hruschka-2024-large}, which will make them favor Option A, we use prompt ``The color of a \{$\text{noun}$\} is \{$\text{color}_{pri}$\} or \{$\text{color}_{sec}$\}?'' as Easy to select $\text{color}_{pri}$ prompt and ``The color of a \{$\text{noun}$\} is \{$\text{color}_{sec}$\} or \{$\text{color}_{pri}$\}?'' as Hard to select $\text{color}_{pri}$ prompt, we can filter the noun-color pairs which LLMs have strong empirical knowledge on. If it constantly selects \{$\text{color}_{pri}$\} instead of the first option, this means that they have a strong empirical bias on the color they chose. Thus, the biased subset can be used to test the performance of patchscopes in biased cases. When we patch these nouns with secondary color information, and the target model successfully reveals this color in the context it generated, this means that the model overcomes empirical bias and effectively translates the hidden representation. 
Table \ref{table:color_dataset} shows the case statistics of each model. 

\begin{table}[t]
\centering
\begin{tabular}{lcc}
\toprule
Model & Biased & Non-biased \\
\midrule
Llama2-7b      & 166 & 980 \\
Llama3.2-1b   & 344 & 802 \\
Llama3-8b      & 483 & 663 \\
Qwen3-4b       & 516 & 630 \\
\bottomrule
\end{tabular}
\caption{Model bias statistics. \textit{Hard} column shows the number of cases where the model chose the primary color under the hard prompt, \textit{Easy} under the easy prompt, and \textit{Bias} under both.}
\label{table:color_dataset}
\end{table}

\subsection{Gender, Culture and Age tasks}\label{app:biased_dataset}
\subsubsection{Dataset Description}
These datapoints are a subset of the dataset provided by \cite{hernandez2024linearity}. The dataset proposed by \cite{hernandez2024linearity} contains four categories:  factual associations, commonsense knowledge, implicit biases, and linguistic knowledge, while we only use the implicit biases subset. In this subset, there are seven relation and for easy statistics, we merge them into three tasks, gender, culture and age. Table \ref{tab:bias_dataset} shows what is contained in each task and also the number of examples per relation.

\begin{table}[]
    \centering
    \begin{tabular}{lll}
        \toprule
         Task& Concept& Number  \\
         \midrule
         \multirow{4}{*}{Gender}
         & degree & 38\\
         & characteristic & 30 \\
         & name & 19 \\
         & occupation & 19 \\
         \midrule
         Birthplace (Culture) & name & 31\\
         Religion (Culture) & name & 31 \\
         \midrule
         Age & occupation & 45 \\
         
         \bottomrule
    \end{tabular}
    \caption{Number of examples per relation in the implicit biases subset. Birthplace and Religion bias on name are merged as Culture bias in the main paper.}
    \label{tab:bias_dataset}
\end{table}

\subsubsection{Source and Target Prompts for Gender Biased Dataset}

For the four subcategories of gender biased tasks, the prompts we use are shown in Table \ref{tab:prompts_for_gender}. For simplify notation, in source prompt, we only denote the gender which model does not toward to as gender$_{sec}$. When implemented, the word can be substituted with the pronoun that corresponds to the gender of the subject. In the occupation and characteristic category, gender$_{sec}$ is man or woman. In the name category, gender$_{sec}$ is She or He. In the degree category, gender$_{sec}$ is male or female. Also, the target prompt $T$ shown in the table has ``She'' as the first option and ``He'' as the second. In experiments, we also forward the counterpart $T'$, where the options' positions are reversed. 

\begin{table}[t]
\centering
\resizebox{\linewidth}{!}{
\small
\begin{tabular}{lp{0.7\linewidth}}
\toprule
Category & Prompts \\
\midrule
\multirow{2}{*}{Occupation} & \textbf{Source: }This is an exceptional \{occupation\} \{gender$_{sec}$\}. \\
& \textbf{Target: } Replace ? with ``She'' or ``He''. ? is an exceptional \{x\}.\\
\midrule
\multirow{2}{*}{Name} & \textbf{Source: }\{Gender$_{sec}$\} has the name \{name\}. \\
& \textbf{Target: }Replace ? with ``She'' or ``He''. ? has the name \{x\}.\\
\midrule
\multirow{2}{*}{Degree} & \textbf{Source: }This is a \{gender$_{sec}$\} outstanding student in department of \{degree\}. \\
& \textbf{Target: }Replace ? with ``She'' or ``He''. ? stands out in the department of \{ x \}.\\
\midrule
\multirow{2}{*}{Characteristic} & \textbf{Source: }This \{gender$_{sec}$\} is a \{characteristic\} person. \\
& \textbf{Target: }Replace ? with ``She'' or ``He''. ? is a \{x\} person.\\
\bottomrule
\end{tabular}}
\caption{The source and target prompts for different categories of gender biased tasks. The designed prompt is to jailbreak the sexism alignment. }
\label{tab:prompts_for_gender}
\end{table}

\subsection{Examples for each task}\label{app:example_task}
In this section, we provide examples of the source and target prompts for each task. The noun we want to explain is highlighted with \underline{underline}.
\begin{itemize}
    \item \textbf{Color:} Source Prompt: \textit{Here is a purple \underline{broccoli}. }Target Prompt: \textit{The color of x is green or purple?}
    \item \textbf{Gender:} Source Prompt: \textit{This man is a \underline{compassionate} person.}
    Target Prompt: \textit{Replace z with He or She. z is a x person.}
    \item \textbf{Culture:} Source Prompt: \textit{Here is a Jewish \underline{Ajahn}.}
    Target Prompt: \textit{The religion of x is Buddhism or Judaism?}
    \item \textbf{Age:} Source Prompt: \textit{Here is an exceptional young \underline{mentor}.}
    Target Prompt: \textit{The age of x is old or young?}
\end{itemize}

\end{document}